\documentclass[a4paper]{article}

\usepackage{mathptmx}
\usepackage{soul}\setuldepth{article}
\usepackage[utf8]{inputenc}
\usepackage{helvet}
\usepackage[overlay,absolute]{textpos}
\usepackage{float}
\usepackage{wrapfig} 
\usepackage{xcolor}
\usepackage{array}
\usepackage{amsmath} 
\usepackage{relsize} 
\usepackage{graphicx}
\usepackage{subcaption}
\usepackage{tikz}
\usepackage{venndiagram} 
\usepackage[none]{hyphenat} 
\usepackage{amsthm}
\usepackage{mathtools}
\usepackage{verbatim}
\usetikzlibrary{arrows.meta}
\usepackage{paralist}
\usepackage{stmaryrd}

\usepackage{multicol} 
\usepackage{vwcol} 

\newtheorem{theorem}{Theorem}[section]

\newtheorem{lemma}[theorem]{Lemma}

\theoremstyle{definition}
\newtheorem{definition}[theorem]{Definition} 
\newtheorem{example}[theorem]{Example} 
\newtheorem{proposition}[theorem]{Proposition} 
\theoremstyle{remark}

\newcommand{\Ar}{\textit{Ar}}
\newcommand{\att}{\rightarrow}
\newcommand{\AFam}{\textit{F}_{\textnormal{ASPIC}-}}
\newcommand{\AFdam}{\textit{J}_{\textnormal{DA}-}}
\newcommand{\CalF}{\mathcal{F}}
\renewcommand{\flat}{\textit{flat}}
\newcommand{\Conc}{\textit{Conc}}
\newcommand{\DefRules}{\textit{DefRules}}
\newcommand{\TopRule}{\textit{TopRule}}
\newcommand{\Sub}{\textit{Sub}}
\newcommand{\supp}{\textit{sup}}

	\makeatletter  
	\newcommand*{\centerfloat}{
  	\parindent \z@
  	\leftskip \z@ \@plus 1fil \@minus \textwidth
  	\rightskip\leftskip
  	\parfillskip \z@skip}
	\makeatother
 
\def\hb{\hbox to 10.7 cm{}}

\newtoggle{onlypaper}
\newtoggle{onlyreport}

\togglefalse{onlypaper}
\toggletrue{onlyreport}

\newcommand{\onlypaper}[1]{\iftoggle{onlypaper}{#1}{}}

\newcommand{\onlyreport}[1]{\iftoggle{onlyreport}{#1}{}}

\begin{document}

\sloppy

\pagestyle{headings}
\def\thepage{}


\title{Technical Report of \\ ``Deductive Joint Support\\ for Rational Unrestricted Rebuttal''}


\author{Marcos Cramer, Meghna Bhadra} 

\date{International Center for Computational Logic, TU Dresden, Germany}

\maketitle

\begin{abstract}
In ASPIC-style structured argumentation an argument can rebut another argument by attacking its conclusion. Two ways of formalizing rebuttal have been proposed: In restricted rebuttal, the attacked conclusion must have been arrived at with a defeasible rule, whereas in unrestricted rebuttal, it may have been arrived at with a strict rule, as long as at least one of the antecedents of this strict rule was already defeasible. One systematic way of choosing between various possible definitions of a framework for structured argumentation is to study what rationality postulates are satisfied by which definition, for example whether the closure postulate holds, i.e.\ whether the accepted conclusions are closed under strict rules. While having some benefits, the proposal to use unrestricted rebuttal faces the problem that the closure postulate only holds for the grounded semantics but fails when other argumentation semantics are applied, whereas with restricted rebuttal the closure postulate always holds. In this paper we propose that ASPIC-style argumentation can benefit from keeping track not only of the attack relation between arguments, but also the relation of deductive joint support that holds between a set of arguments and an argument that was constructed from that set using a strict rule. By taking this deductive joint support relation into account while determining the extensions, the closure postulate holds with unrestricted rebuttal under all admissibility-based semantics. We define the semantics of deductive joint support through the flattening method.
\end{abstract}


\section{Introduction}

Formal argumentation has become a fruitful field of research within AI~\cite{rahwan2009argumentation}. It comprises two main branches: Abstract argumentation is based on the idea promoted by Dung~\cite{dung1995acceptability} that under some conditions, the acceptance of arguments depends only on the {\em attack relation} between the arguments, i.e.\ on the relation that holds between a counterargument and the argument that it counters. This idea gives rise to the notion of an \emph{argumentation framework} (\emph{AF}), a directed graph whose nodes represent arguments and whose edges represent the attack relations between them, as well as to the notion of an \emph{argumentation semantics}, a way of choosing accepted arguments from an argumentation framework. Structured argumentation, on the other hand, studies the internal structure of arguments that are constructed in some logical language and specifies how this internal structure determines the attack relation between the arguments. Once the attack relation has been specified, the argumentation semantics from abstract argumentation can be applied to determine the acceptability of arguments.

One important approach within structured argumentation is that of ASPIC-style frameworks like ASPIC+~\cite{modgil2014aspic+} and ASPIC$-$~\cite{caminada2014preferences}, in which arguments are constructed by applying strict and defeasible rules to strict and defeasible premises. In these ASPIC-style frameworks one can distinguish various kinds of attacks depending on which part of an argument gets questioned. One kind of attack is a rebuttal, in which one argument attacks the conclusion of another argument. Two ways of formalizing rebuttal have been proposed: ASPIC+ makes use of \emph{restricted rebuttal}, in which the attacked conclusion must have been arrived at with a defeasible rule, whereas ASPIC$-$ makes use of \emph{unrestricted rebuttal}, in which the attacked conclusion may have been arrived at with a strict rule, as long as at least one of the antecedents of this strict rule was already defeasible. 

One systematic way of choosing between various possible definitions of a framework for structured argumentation is to study what rationality postulates are satisfied by which definition \cite{caminada2017rationality}. One example of such a rationality postulate is the closure postulate, according to which the accepted conclusions should be closed under strict rules, i.e.\ a statement that is derivable by applying a strict rule to some accepted conclusion should itself be an accepted conclusion. For all admissibility-based argumentation semantics, e.g.\ grounded, complete, stable or preferred semantics, ASPIC+ satisfies the closure postulate. ASPIC$-$ on the other hand only satisfies closure under the grounded semantics, but fails to do so for the others, such as the preferred semantics. This failure of the closure postulate is due to the use of unrestricted rebuttal. On the other hand, from the point of view of human argumentation, unrestricted rebuttal seems to be a very natural way of attacking an argument. This intuition has also been underpinned through an empirical study of human evaluation of arguments~\cite{yu2018structured}.

In this paper we propose a modification to ASPIC$-$, called \emph{Deductive ASPIC$-$}, that ensures that the closure postulate is satisfied under all admissibility-based argumentation semantics. The underlying idea is to keep track not only of the attack relation between arguments, but also the relation of deductive joint support that holds between a set of arguments and an argument that was constructed from that set using a strict rule. For this purpose, we introduce the notion of a \emph{Joint Support Bipolar Argumentation Framework} (\emph{JSBAF}), which contains an attack relation like usual AFs and additionally a joint support relation, whose intuitive interpretation is a deductive support from the supporting arguments towards the supported arguments due to the latter being constructed by applying a strict rule to the former. We show how existing argumentation semantics for AFs can be adapted to semantics for JSBAFs using the flattening method. We prove that the resulting framework for structured argumentation satisfies closure as well as two other important rationality postulates, direct consistency and indirect consistency. In this paper we limit ourselves to structured argumentation without preferences, leaving the generalization of the results to preference-based argumentation for future work.

The paper is structured as follows: Section~\ref{sec:prelim} contains required preliminaries from abstract and structured argumentation. In Section~\ref{sec:JSBAF} we define JSBAFs and show how existing argumentation semantics for AFs can be adapted to semantics for JSBAFs using the flattening method. In Section~\ref{sec:postulates} we define Deductive ASPIC$-$, prove that it satisfies the closure postulate and the two consistency postulates, and finally illustrate the functioning of Deductive ASPIC$-$ by adapting Caminada's tandem example~\cite{caminada2017rationality} to Deductive ASPIC$-$. Section~\ref{sec:conclusion} concludes the paper and presents avenues for further research.

\section{Preliminaries of Abstract and Structured Argumentation}
\label{sec:prelim}

This section briefly presents some required preliminaries of abstract and structured argumentation, starting with the notion of an \emph{argumentation framework} due to Dung~\cite{dung1995acceptability}.

\begin{definition}
 An \emph{argumentation framework (AF)} $F = (\Ar,\att)$ is a (finite or infinite) directed graph in which the set $\Ar$  of vertices is considered to represent arguments and the set ${\att} \subseteq \Ar \times \Ar$ of edges is considered to represent the attack relation between arguments, i.e.\ the relation between a counterargument and the argument that it counters.
\end{definition}


Given an argumentation framework, we want to choose sets of arguments for which it is rational and coherent to accept them together. Such a set of arguments that may be accepted together is called an \emph{extension}. Multiple \emph{argumentation semantics} have been defined in the literature, i.e.\ multiple different ways of defining extensions given an argumentation framework. Before we consider specific argumentation semantics, we first give a formal definition of the notion of an \emph{argumentation semantics}:

\begin{definition}
 An \emph{argumentation semantics} is a function $\sigma$ that maps any AF ${F = (\Ar,\att)}$ to a set $\sigma(F) \subseteq 2^\Ar$. The elements of $\sigma(F)$ are called $\sigma$-extensions of~$F$.
\end{definition}


In this paper we consider the complete, stable, grounded and preferred semantics:

\begin{definition}
Let $F = (\Ar,\att)$ be an AF, and let $S \subseteq \Ar$. The set $S$ is called \emph{conflict-free} iff there are no arguments $b,c \in S$ such that $b$ attacks $c$ (i.e. such that $(b,c) \in \att$). 
Argument $a \in \Ar$ is \emph{defended} by $S$ iff for every $b \in \Ar$ such that $b$ attacks $a$ there exists $c \in S$ such that $c$ attacks $b$. We say that $S$ is \emph{admissible} iff $S$ is conflict-free and every argument in $S$ is defended by $S$.
\vspace{-2.5mm}
\begin{itemize}
    \setlength{\itemsep}{1pt}%
    \setlength{\parskip}{0pt}%
\item $S$ is a \emph{complete extension} of $F$ iff $S$ is admissible and $S$ contains all the arguments it defends. 
\item $S$ is a \emph{stable extension} of $F$ iff $S$ is admissible and $S$ attacks all arguments in $\Ar \setminus S$.
\item $S$ is the \emph{grounded extension} of $F$ iff $S$ is the minimal (with respect to set inclusion) complete extension of $F$. 
\item $S$ is a \emph{preferred extension} of $F$ iff $S$ is a maximal (with respect to set inclusion) complete extension of $F$.
\end{itemize}
\end{definition}
All of these four semantics satisfy the property that every extension is an admissible set. Due to this property they are called \emph{admissibility-based semantics}.

We now turn towards the definition of ASPIC$-$, a framework for structured argumentation introduced by~\cite{caminada2014preferences}.
\begin{definition} 
  Given a logical language $L$ that is closed under negation $(\neg)$, an \emph{argumentation system} over $L$ is a tuple ${AS = (R_s,R_d,n)}$ where:
\vspace{-2.5mm}
	\begin{itemize}
    \setlength{\itemsep}{1pt}%
    \setlength{\parskip}{0pt}%
		\item $R_s$ is a finite set of strict inference rules of the form $\varphi_1,\ldots,\varphi_k \mapsto \varphi$ where $\varphi_i,\varphi$ are elements in $L$ and $k\geq 0$.
		\item $R_d$ is a finite set of defeasible inference rules of the form $\varphi_1,\ldots,\varphi_k \Mapsto \varphi$ where $\varphi_i,\varphi$ are elements in $L$ and $k\geq 0$.
		\item $n$ is a partial function such that $n: R_d \rightarrow L$.
	\end{itemize}
\end{definition}

\begin{definition}
 Let $\varphi$ and $\psi$ be formulas. $\varphi = -\psi$ means that $\varphi = \neg \psi$ or $\psi = \neg \varphi$.
\end{definition}


\begin{definition} 
An \emph{argument $A$ on the basis of} an argumentation system ${AS = (R_s,R_d,n)}$ is defined recursively as follows:
\vspace{-2mm}
\begin{itemize}
    \setlength{\itemsep}{1pt}%
    \setlength{\parskip}{0pt}%
 \item $A_1,\ldots,A_n\mapsto\psi$ is an argument if $A_1,\ldots,A_k (k\geq0)$ are arguments, and there is a strict rule $r = \Conc(A_1),\ldots,\Conc(A_n)\mapsto\psi$ in $R_s$. In that case $\DefRules(A) \coloneqq \DefRules(A_1)\cup\ldots\cup \DefRules(A_k)$.
 \item $A_1,\ldots,A_n\Mapsto\psi$ is an argument if $A_1,\ldots,A_k (k\geq0)$ are arguments, and there is a defeasible rule $r = \Conc(A_1),\ldots,\Conc(A_n)\Mapsto\psi$ in $R_d$. In that case $\DefRules(A) \coloneqq \DefRules(A_1)\cup\ldots\cup \DefRules(A_k)\cup\{r\}$.
\end{itemize}
\vspace{-2mm}
In both cases we define $\Conc(A) \coloneqq \psi$, $\Sub(A) \coloneqq \Sub(A_1)\cup\ldots\cup \Sub(A_k)\cup\{A\}$ and $\TopRule(A) \coloneqq r$. Furthermore, we call an argument $A$ \emph{defeasible} iff $\DefRules(A) \neq \emptyset$.
\end{definition}

\begin{definition}
An argumentation system $AS$ is called \emph{consistent} iff there are no strict arguments $A,B$ on the basis of $AS$ such that $\Conc(A)= -\Conc(B)$.
\end{definition}

\begin{definition}
Let $A$ and $B$ be arguments on the basis of an argumentation system ${AS = (R_s,R_d,n)}$. We say that
		$A$ \emph{undercuts} $B$ (on $B^\prime$) iff $\Conc(A) = -n(r)$ for some $B^\prime \in \Sub(B)$ with $\TopRule(B^\prime) = r$ and $r \in R_d$.
		We say that $A$ \emph{unrestrictedly rebuts} $B$ (on $B^\prime$) iff $\Conc(A) = -\Conc(B^\prime)$ for some defeasible $B^\prime\in \Sub(B)$.
\end{definition}

\begin{definition}
 Let ${AS = (R_s,R_d,n)}$ be an argumentation system. The \emph{argumentation framework corresponding to $AS$ according to ASPIC$-$}, denoted by $\AFam(AS)$, is the AF $(\Ar,\att)$, where $\Ar$ is the set of arguments on the basis of $AS$ and $A \att B$ holds whenever $A$ undercuts or unrestrictedly rebuts $B$. 
\end{definition}


\begin{definition}
 Let ${AS = (R_s,R_d,n)}$ be an argumentation system and let $\sigma$ be an argumentation semantics. For every $\sigma$-extension $E$ of $\AFam(AS)$, the set $\{\Conc(A) \mid A \in E\}$ is called a \emph{set of ASPIC$-$ conclusions of $AS$ under $\sigma$}.
\end{definition}

This concludes the definition of the structured argumentation framework ASPIC$-$. Finally we define three rationality postulates due to Caminada and Amgoud~\cite{caminada2007evaluation} that structured argumentation frameworks should ideally satisfy. 

\begin{definition}
 Let $L$ be a logical language, let ${AS = (R_s,R_d,n)}$ be an argumentation system over $L$, and let $S \subseteq L$. The \emph{closure of $S$ under strict rules}, denoted $Cl_{R_s}(S)$, is the smallest set such that $Cl_{R_s}(S) \supseteq S$ and for every strict rule $a_1,\ldots,a_n \mapsto b \in R_s$ such that $a_1,\ldots,a_n \in Cl_{R_s}(S)$, we have $b \in Cl_{R_s}(S)$.
\end{definition}

\begin{definition}
	Let $\CalF$ be a framework for structured argumentation (e.g.\ ASPIC$-$) and let $\sigma$ be an argumentation semantics.
\vspace{-2mm}
	\begin{itemize}
	 \item $\CalF$ satisfies \emph{closure} under $\sigma$ iff for every consistent argumentation system $AS$ and for every set $C$ of $\CalF$-conclusions of $AS$ under $\sigma$, we have $Cl_{R_s}(C) = C$. 
	 \item $\CalF$ satisfies \emph{direct consistency} under $\sigma$ iff for every consistent argumentation system $AS$ and every set $C$ of $\CalF$-conclusions of $AS$ under $\sigma$, there is no $\phi$ such that $\phi,\neg\phi \in C$.
	 \item $\CalF$ satisfies \emph{indirect consistency} under $\sigma$ iff for every consistent argumentation system $AS$ and every set $C$ of $\CalF$-conclusions of $AS$ under $\sigma$, there is no $\phi$ such that $\phi,\neg\phi \in Cl_{R_s}(C)$.
	\end{itemize}
\end{definition}

Caminada et al.~\cite{caminada2014preferences} showed that ASPIC$-$ satisfies these three postulates under the grounded semantics, whereas Caminada and Wu~\cite{caminada2011limitations} showed that closure and indirect consistency are violated by ASPIC$-$ under the preferred semantics.

\section{Abstract Argumentation with Deductive Joint Support}
\label{sec:JSBAF}

Multiple ways of augmenting argumentation frameworks with a support relation between arguments have been considered in the literature \cite{boella2010support,nouioua2011argumentation,oren2008semantics}, giving rise to a lively research field called \emph{bipolar argumentation}. One of the ways of formally interpreting the support relation is called \emph{deductive support} \cite{boella2010support}. Here the idea is that when an argument $a$ deductively supports an argument $b$, this is similar to the situation when a formula $\phi$ logically entails a formula $\psi$, i.e.\ if one accepts $a$, one should also accept $b$. 

In this section we extend the deductive support relation to a \emph{deductive joint support} relation, in which multiple arguments together can deductively support another argument. Again the intuitive idea is similar to when multiple formulas entail another formula: When $a_1, \dots, a_n$ jointly deductively support $b$, this means that if $a_1, \dots, a_n$ are all accepted, $b$ should also be accepted. In this section we introduce the notion of a \emph{Joint Support Bipolar Argumentation Framework} (JSBAF) in which such joint support relations appear alongside the usual attack relation. We then show how the flattening methodology (see \cite{boella2009meta}) can be used to adapt standard argumentation semantics to semantics for JSBAFs that give a deductive interpretation to the joint support relation.

\begin{definition} 
A \emph{Joint Support Bipolar Argumentation Framework} (JSBAF) is a triple ${(Ar,\rightarrow ,\Rightarrow)}$ such that $Ar$ is the set of all arguments in the framework, ${\rightarrow\subseteq Ar\times Ar}$ is an attack relation and ${\Rightarrow\subseteq2^{Ar}\times Ar}$ is a joint support relation.
	\label{def:jsbaf}
\end{definition}

\begin{example} \label{fig:jsbafframework}
	As an example, we illustrate below a JSBAF $J_1$ in which arguments $a$ and $b$ jointly support $c$ which is attacked by argument $d$:
	
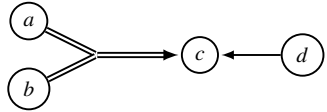
\vspace{-3mm}
	\begin{figure}[H]
		\begin{center}
		\scalebox{0.90}[0.90]{
			\begin{tikzpicture}[->,>=stealth,shorten >=1pt,node distance=1.5cm,thick,main node/.style={circle,draw,font=\small}]
				\node[main node] (1) at (0,0) {$c$};
				\coordinate (5) at (-1.5,0) {};
				\node[main node] (2) at (-2.5,0.5){$a$};
				\node[main node] (3) at (-2.5,-0.5) {$b$};
				\node[main node] (4) at (1.5,0) {$d$};
				\draw[double distance=1pt,-,>=latex] (2) to (5);
				\draw[double distance=1pt,-,>=latex] (3) to (5);
				\draw[double distance=1pt,->,>=latex] (5) to (1);
				\draw[->,>=latex] (4) to (1);
			\end{tikzpicture}
		}
		\end{center}
	\end{figure}
	\vspace{-8mm}
\end{example}

Now that we have formally defined a JSBAF, let us move on to its semantics. The principle idea is the same as in argumentation frameworks.

\begin{definition} 
  A JSBAF Semantics is a function that maps every JSBAF $J = (Ar,\rightarrow ,\Rightarrow)$ to a set $\sigma(J) \subseteq {2^{Ar}}$. The elements of $\sigma(J)$ are called $\sigma$-extensions. 
\end{definition}

The deductive property of the joint support relation inspires the following notion of a deductive JSBAF semantics:

\begin{definition} 
  A JSBAF semantics $\sigma$ is called \emph{deductive} iff for every JSBAF $J = {(Ar,\rightarrow ,\Rightarrow)}$, for every $\sigma$-extension $E$ of $J$, every set ${S\subseteq E}$ and every ${A\in Ar}$ such that ${S\Rightarrow A}$, we have ${A\in E}$.	
\end{definition}	

Furthermore, we will require the notion of a conflict-free JSBAF semantics:

\begin{definition} 
  A JSBAF semantics $\sigma$ is called \emph{conflict-free} iff for every JSBAF $J = {(Ar,\rightarrow ,\Rightarrow)}$, for every $\sigma$-extension $E$ and any $a,b \in E$, $(a,b) \notin \rightarrow$.
\end{definition}	

We will flatten JSBAFs to standard AFs in a two-step process. In the first step, we flatten JSBAFs to \emph{higher-level argumentation frameworks} (originally introduced by Gabbay~\cite{gabbay2009fibring}) that contain joint attacks.


\begin{definition} 
  \label{def:jointattack} 
  We define a \textit{higher-level argumentation framework} (\emph{higher level AF}) as a tuple $(Ar,\rightarrow)$ where $Ar$ is the set of arguments in the framework and $\rightarrow \subseteq (2^{Ar}\setminus \{\emptyset\}) \times Ar$ is a joint attack relation.
\end{definition}


\begin{example} \label{fig:jattack}
	The following is an example of a higher-level AF, where arguments $a$ and $b$ jointly attack $c$:
		
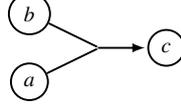
\vspace{-7mm}
	\begin{figure}[H]
		\begin{center}
		\scalebox{0.90}[0.90]{
			\begin{tikzpicture}[->,>=stealth,shorten >=1pt,auto,node distance=1.5cm,thick,main node/.style={circle,draw,font=\small}]
				\node[main node] (1) at (0,0) {$c$};
				\coordinate (2) at (-1,0) {};
				\coordinate (3) at (-0.98,0) {};
				\node[main node] (4) at (-2,0.5){$b$};
				\node[main node] (5) at (-2,-0.5) {$a$};
				\draw[->,>=latex] (2) to (1);
				\draw[-,>=latex] (4) to (3);
				\draw[-,>=latex] (5) to (3);
			\end{tikzpicture}
		}
		\end{center}
		\vspace{-6mm}
	\end{figure}
\end{example}

We will now define the two-step process of flattening a JSBAF into a standard AF. The first step involves transforming the JSBAF to a higher-level AF by converting joint supports to joint attacks while introducing some meta-arguments. The second step involves transforming the higher-level AF to a standard AF by converting joint attacks to regular one-on-one attacks.

\begin{definition} 
Let  $J = {(Ar,\rightarrow ,\Rightarrow)}$ be a JSBAF. The one-step flattening of $J$, denoted by $\flat_1(J)$, is a higher-level AF ${(MS,\rightarrow_1)}$, where $MS \coloneqq \Ar \cup {\{\bar{b} \mid (X,b)\in\Rightarrow\}}$ and the joint attack relation $\rightarrow_1$ is defined as follows:
\vspace{-2mm}
	\begin{itemize}
			\item For each $(a,b) \in \rightarrow$, we have $(a,b) \in \rightarrow_1$.
			\item For each $(X,b)\in\Rightarrow$, we have $b \rightarrow_1 \bar{b}$.
			\item For each $X$ with ${(X,b)\in\Rightarrow}$ and for every ${a\in X}$, we have ${(X\setminus \{a\}) \cup \{\bar{b}\} \rightarrow_1 a}$.
	\end{itemize}
\end{definition}

In what follows, we present three examples of this flattening: 
\begin{example} \label{ex:singlesup}
On the left, Figure~\ref{fig:flatjsupwithonearg} depicts a JSBAF $J_2$ consisting of an argument $a$ supporting another argument $b$. On the right, Figure~\ref{fig:flatjsupwithonearg} depicts its one-step flattening $\flat_1(J_2)$.
\vspace{-3mm}
	\begin{figure}[h] 
		\captionsetup[subfigure]{justification=justified,singlelinecheck=false}
		\centering 
  		\begin{subfigure}[b]{0.4\linewidth}
  		\scalebox{0.90}[0.90]{
   		\begin{tikzpicture}[->,>=stealth,shorten >=1pt,node distance=1.5cm,double distance=4pt,thick,main node/.style={circle,draw,font=\small}]
			\node[main node] (1) at (4,0) {$b$};
			\node[main node] (2) at (1,0){$a$};
			\draw[double distance=1pt,->,>=latex] (2) to (1);
		\end{tikzpicture} 
		}
  		\end{subfigure}
  		\hspace{17mm}
		\begin{subfigure}[b]{0.4\linewidth}
		\scalebox{0.90}[0.90]{
 		\begin{tikzpicture}[->,>=stealth,shorten >=1pt,auto,node distance=1.5cm,thick,main node/.style={circle,draw,font=\small}]
			\node[main node] (2) at (0,0) {$\bar{b}$};
			\node[main node] (1) at (2,0){$b$};
			\node[main node] (3) at (-2,0) {$a$};
			\draw[->,>=latex] (1) to (2);
			\draw[->,>=latex] (2) to (3);
		\end{tikzpicture}
		}
		\end{subfigure}
		\vspace{-2mm}
		\caption{JSBAF $J_2$ and its one-step flattening $\flat_1(J_2)$} \label{fig:flatjsupwithonearg} 
	\end{figure}
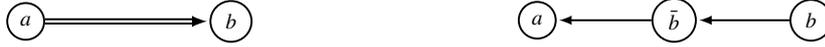
	\vspace{-3mm}
\end{example}

\begin{example} \label{ex:jsupwithtwoargs}
	As a second example we reconsider the JSBAF $J_1$ from Example \ref{fig:jsbafframework}. 
	The following is its one-step flattening $\flat_1(J_1)$:
	
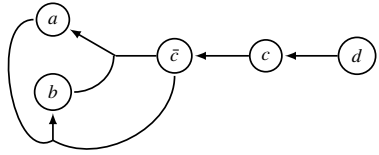
\vspace{-3mm}
	\begin{figure}[H]
		\begin{center}
			\scalebox{0.8}[0.8]{
			\begin{tikzpicture}[->,>=stealth,shorten >=1pt,auto,node distance=1.5cm,thick,main node/.style={circle,draw,font=\small}]
				\node[main node] (1) at (-2.5,0.6){$a$};
				\node[main node] (2) at (-2.5,-0.6) {$b$};
				\node[main node] (3) at (-0.5,0) {$\bar{c}$};
				\node[main node] (4) at (1,0) {$c$};
				\coordinate (5) at (-1.5,0) {};
				\coordinate (6) at (-1.47,0) {};
				\coordinate (7) at (-2.5,-1.4) {};
				\coordinate (8) at (-2.5,-1.4) {};
				\node[main node] (9) at (2.5,0) {$d$};
				\draw[-,>=latex] (3) to (5); 
				\draw[-,>=latex] (5) to[out=-90,in=0] (2); 
				\draw[->,>=latex] (6) to (1);
				\draw[->,>=latex] (8) to (2);
				\draw[-,>=latex] (8) to[out=-30,in=-90] (3);
				\draw[-,>=latex] (8) to[out=-150,in=180] (1);
				\draw[->,>=latex] (4) to (3);
				\draw[->,>=latex] (9) to (4);						
			\end{tikzpicture} 
			}
		\end{center}
		\vspace{-6mm}
		\label{fig:flatjsupwithtwoargs}
	\end{figure}
\end{example}

\begin{example} \label{fig:jsupwiththreeargs}
	As a last example, we illustrate a JSBAF $J_3$ where arguments \{$a$, $b$, $c$\} jointly support $d$ (on the left) as well as its one-step flattening $\flat(J_3)$ (on the right):	
	\vspace{-8mm}
	\begin{figure}[H] 
		\centering
  		\begin{subfigure}[b]{0.4\linewidth}
  			\captionsetup{justification=justified,singlelinecheck=false,skip=25pt}
  			\scalebox{0.95}[0.95]{
   			\begin{tikzpicture}[->,>=stealth,shorten >=1pt,node distance=1.5cm,thick,main node/.style={circle,draw,font=\small}]
				\node[main node] (1) at (0,0) {$d$};
				\coordinate (2) at (-1,0) {};
				\node[main node] (3) at (-2.5,0){$b$};
				\node[main node] (4) at (-2.5,1){$a$};
				\node[main node] (5) at (-2.5,-1) {$c$};
				\coordinate (6) at (-0.98,0) {};
				\draw[double distance=1pt,->,>=latex] (2) to (1);
				\draw[double distance=1pt,-,>=latex] (3) to (6);
				\draw[double distance=1pt,-,>=latex] (4) to (6);
				\draw[double distance=1pt,-,>=latex] (5) to (6);
			\end{tikzpicture}
			}
  		\end{subfigure}
  		\hspace{9mm}
		\begin{subfigure}[b]{0.4\linewidth}
		\captionsetup{justification=justified,singlelinecheck=false}
 		\scalebox{0.75}[0.75]{
			\begin{tikzpicture}[->,>=stealth,shorten >=1pt,node distance=1.5cm,thick,main node/.style={circle,draw,font=\small}]
			\node[main node] (1) at (-2.5,1.5) {$a$};
			\node[main node] (2) at (-2.5,0) {$b$};
			\node[main node] (3) at (-2.5,-1.5) {$c$};
			\node[main node] (4) at (0,0) {$\bar{d}$};
			\node[main node] (5) at (1.5,0) {$d$};
			\coordinate (6) at (-1.7,0) {};
			\coordinate (7) at (-3.5,1.5) {};
			\coordinate (9) at (-3.49,1) {};
			\coordinate (8) at (-3.5,-1.5) {};
			\coordinate (10) at (-3.5,-1.05) {};
			\draw[->,>=latex] (5) to (4);
			\draw[-,>=latex] (6) to (4);
			\draw[->,>=latex] (6) to (2);
			\draw[-,>=latex] (3) to[out=0,in=-3] (6);
			\draw[-,>=latex] (1) to[out=0,in=3]  (6);
			\draw[->,>=latex] (7) to (1);
			\draw[-,>=latex] (7) to[out=-120,in=180] (2);
			\draw[-,>=latex] (7) to[out=-150,in=150] (3);
			\draw[-,>=latex] (7) to[out=120,in=90] (4);
			\draw[-,>=latex] (8) to[out=150,in=-150] (1);
			\draw[-,>=latex] (8) to[out=120,in=-180] (2);
			\draw[->,>=latex] (8) to (3);
			\draw[-,>=latex] (8) to[out=-120,in=-90] (4);
			\end{tikzpicture} 
			}
  		\vspace{-11mm}
		\label{fig:flatthreesup}
		\end{subfigure}
	\end{figure}
\end{example}

In the next step, we define how a higher-level AF can be flattened to a standard AF. This flattening is originally due to Gabbay~\cite{gabbay2009fibring}.
\begin{definition} 
Let $H = (Ar,\rightarrow)$ be a higher-level AF. The flattening of $H$, denoted by $\flat_2(H)$, is a standard argumentation framework ${(MS,\rightarrow_2)}$, where the set $MS$ of meta-arguments and the attack relation $\rightarrow_2$ are defined as follows:
\vspace{-1mm}
\begin{itemize}
 \item $MA = \Ar \cup \{\bar{a}\mid$ there exists an attack $(X,b)\in\rightarrow$ with $a\in X$ and $\vert{X}\vert>1\}\cup\{e(X)\mid$ there exists an attack $(X,b)\in\rightarrow$ with $\vert{X}\vert>1\}$.
 \item For all $(\{a\},b)\in\rightarrow$, we have $a \rightarrow_2 b$.
 \item For all $(X,b)\in\rightarrow$ where $\vert{X}\vert>1$, we have $e(X) \rightarrow_2 b$, and for every $a\in X$, we have $a\rightarrow_2 \bar{a}$ and $\bar{a}\rightarrow_2 e(X)$.
\end{itemize}
\end{definition}

\begin{example}
 The following figure depicts the flattening $\flat_2(\flat_1(J_1))$ of the higher-level AF $\flat_1(J_1)$ depicted in Example~\ref{ex:jsupwithtwoargs}. This example also illustrates how to combine the flattenings $\flat_1$ and $\flat_2$ in order to flatten a JSBAF a standard AF in two flattening steps.
	\vspace{-3mm}
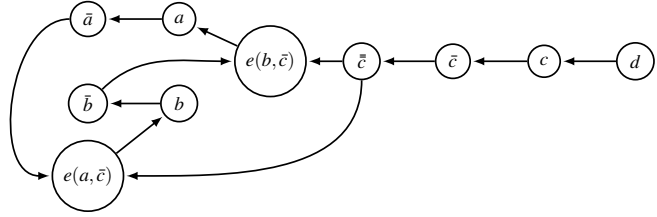
\begin{figure}[H]
	\begin{center}
	\scalebox{0.8}[0.8]{
		\begin{tikzpicture}[->,>=stealth,shorten >=1pt,auto,node distance=1.5cm,thick,main node/.style={circle,draw,font=\small}]
			\node[main node] (1) at (-4.5,0.7){$\bar{a}$};
			\node[main node] (2) at (-3,0.7) {$a$};
			\node[main node] (3) at (-3,-0.7) {$b$};
			\node[main node] (4) at (-4.5,-0.7) {$\bar{b}$};
			\node[main node] (5) at (0,0) {$\bar{\bar{c}}$};
			\node[main node] (6) at (1.5,0) {$\bar{c}$};
			\node[main node] (7) at (3,0) {$c$};
			\node[main node] (8) at (4.5,0) {$d$};

			\node[main node] (9) at (-1.5,0) {$e(b,\bar{c})$};
			\node[main node] (10) at (-4.5,-1.9) {$e(a,\bar{c})$};

			\draw[->,>=latex] (8) to (7); 
			\draw[->,>=latex] (7) to (6); 
			\draw[->,>=latex] (6) to (5);

			\draw[->,>=latex] (5) to (9);
			\draw[->,>=latex] (4) to[out=45,in=180] (9);
			\draw[->,>=latex] (9) to (2);

			\draw[->,>=latex] (2) to (1);
			\draw[->,>=latex] (1) to[out=180,in=180] (10);
			\draw[->,>=latex] (5) to[out=-90,in=0] (10);	
			\draw[->,>=latex] (3) to (4);
			\draw[->,>=latex] (10) to (3);
		\end{tikzpicture} 
		}
	\end{center}
	\vspace{-5mm}
	\caption{The flattening $\flat_2(\flat_1(J_1))$ of the higher-level AF $\flat_1(J_1)$ depicted in Example~\ref{ex:jsupwithtwoargs}}
	\label{fig:jattack}
\end{figure}
\vspace{-2mm}
\end{example}
During the two-step process of flattening a JSBAF framework to an abstract argumentation framework, meta-arguments like the $\bar{c}$ and the $\bar{\bar{c}}$ were introduced. As shown in Figure~\ref{fig:jattack}, $\bar{\bar{c}}$ is attacked only by $\bar{c}$ which is in turn attacked only by $c$. So intuitively, if $c$ is accepted $\bar{\bar{c}}$ should also be accepted. Similarly if argument $c$ is rejected, $\bar{\bar{c}}$ should also be rejected. As a result both $\bar{c}$ and $\bar{\bar{c}}$ can be omitted, thus simplifying the flattened framework. This inspires the following definition for the simplified flattening of a JSBAF to an AF:

\begin{definition} 	\label{def:flatjsbaffw}
Let ${J=(Ar,\rightarrow,\Rightarrow)}$ be a JSBAF. We write $(MA',\rightarrow')$ for $\flat_2(\flat_1(J))$. Then the simplified flattening of $J$, denoted by $\flat(J)$, is a standard argumentation framework ${(MA^*,\rightarrow^*)}$ defined as follows: 
\vspace{-1mm}
\begin{itemize}
 \item $MA^* \coloneqq MA' \setminus \{\bar a, \bar {\bar a} \mid \textnormal{there is some } (S,a) \in \Rightarrow \textnormal{ with } |S| > 1 \}$
 \item $\rightarrow^* \coloneqq \{ (a,b) \in MA^* \times MA^* \mid a \rightarrow' b \} \cup \{ (a,b) \in MA^* \times MA^* \mid \bar {\bar a} \rightarrow' b \}$
\end{itemize}
\end{definition}

\begin{example}
The following figure depicts the simplified flattening $\flat(J_1)$ of the JSBAF $J_1$ from Example~\ref{fig:jsbafframework}:
	\vspace{-1mm}
\begin{figure}[H]
	\begin{center}
		\scalebox{0.80}[0.80]{
		\begin{tikzpicture}[->,>=stealth,shorten >=1pt,auto,node distance=1.5cm,thick,main node/.style={circle,draw,font=\small}]
			\node[main node] (1) at (-4.5,0.7){$\bar{a}$};
			\node[main node] (2) at (-3,0.7) {$a$};
			\node[main node] (3) at (-3,-0.7) {$b$};
			\node[main node] (4) at (-4.5,-0.7) {$\bar{b}$};
			\node[main node] (5) at (0.5,0) {$c$};
			\node[main node] (6) at (2.5,0) {$d$};
			\node[main node] (9) at (-1.5,0) {$e(b,c)$};
			\node[main node] (10) at (-4.5,-1.9) {$e(a,c)$};
			\draw[->,>=latex] (6) to (5);
			\draw[->,>=latex] (5) to (9);
			\draw[->,>=latex] (4) to[out=45,in=180] (9);
			\draw[->,>=latex] (9) to (2);
			\draw[->,>=latex] (2) to (1);
			\draw[->,>=latex] (1) to[out=180,in=180] (10);
			\draw[->,>=latex] (5) to[out=-90,in=0] (10);	
			\draw[->,>=latex] (3) to (4);
			\draw[->,>=latex] (10) to (3);	
		\end{tikzpicture}
		}
	\end{center}
	\vspace{-6mm}
\end{figure} 
\end{example}

The flattening function $\flat$ allows us to define semantics of JSBAFs by first applying the flattening function and then applying a standard argumentation semantics:
\begin{definition}
\label{def:supsem}
	For any Dung semantics $\sigma$, we define a JSBAF semantics $\supp(\sigma)$ as follows: $E$ is a ${\supp(\sigma)}$-extension of ${(Ar,\rightarrow,\Rightarrow)}$ iff there is an extension $E^\prime$ of ${\flat(Ar,\rightarrow,\Rightarrow)}$ such that ${E=E^\prime\cap Ar}$.
\end{definition}

The following two lemmas establish useful facts about $\supp(\sigma)$:
\begin{lemma}
	For any admissibility-based semantics {$\sigma$} of Dung's abstract frameworks, {$\supp(\sigma)$} is a deductive JSBAF semantics.
	\label{lem:dungjsbafsem}
\end{lemma}

\begin{proof}
	Let {$J=(Ar,\rightarrow,\Rightarrow)$} be a JSBAF and let $E$ be a $\supp(\sigma)$-extension of $J$. By Definition \ref{def:supsem}, there is a $\sigma$-extension {$E^\prime$} of $\flat(J) = (MA^*,\att^*)$ such that {$E=E^\prime \cap Ar$}. Let $S \subseteq E$ and $d \in Ar$ be such that $S \Rightarrow d$. We need to show that {$d\in E$}. We distinguish two cases:
	
	\noindent \underline{Case 1: $|S| = 1$, say $S = \{a\}$.} In this case, in $\flat(J)$, $S \Rightarrow d$ is flattened to the attacks $d \att^* \bar d$ and $\bar d \att^* a$. Since $a \in E^\prime$ and $E^\prime$ is admissible, $E^\prime$ defends $a$, i.e.\ $E^\prime$ attacks $\bar d$. But $d$ is the only attacker of $\bar d$ in $\flat(J)$, so $d \in E^\prime$, i.e.\ $d \in E$.

	\noindent
	\underline{Case 2: $|S| > 1$.} In this case, in $\flat(J)$, $S \Rightarrow d$ is flattened into the following attacks for each element {$a$} in {$S$}:
	
	\vspace{-4mm}
\setlength\columnsep{1pt}
\noindent \parbox{1.2\linewidth}{
	\begin{multicols}{2}
	\vspace{-4mm}
		\begin{itemize}
		\item {$a\rightarrow^* \bar{a}$}
		\item for each {$b \in S\setminus \{a\}$}: {$\bar{b} \rightarrow^* e(\{d\}\cup (S\setminus \{a\}))$}
		\item {$e(\{d\}\cup (S\setminus \{a\}))\rightarrow^* a$}
		\item {$d\rightarrow^* e(\{d\}\cup (S\setminus \{a\}))$}
	\end{itemize}
	\end{multicols}
	\vspace{-2mm}
	}
	Let $a\in S$. Then $a \in E^\prime$, i.e.\ $E^\prime$ defends $a$. 
	Since {$e(\{d\}\cup (S\setminus \{a\}))\rightarrow^* a$}, so {$E^\prime$} attacks {$e(\{d\}\cup (S\setminus \{a\}))$}.
	However, for each element {$b$} in {$(S\setminus \{a\})$}, {$b\in E^\prime$}, so by the conflict-freeness of $E^\prime$, {$\bar{b} \notin E^\prime$}. So the only element of {$E^\prime$} that can attack {$e(\{d\}\cup (S\setminus \{a\}))$} is {$d$}. So {$d\in E^\prime$}, i.e.\ {$d\in E$}.
\end{proof}

\begin{lemma}
	For any admissibility-based semantics {$\sigma$} of Dung's abstract frameworks, {$\supp(\sigma)$} is a conflict-free JSBAF semantics.
	\label{lem:jsbafcfreesem}
\end{lemma}

\begin{proof}
	Let {$J=(Ar,\rightarrow,\Rightarrow)$} be a JSBAF and let $E$ be a $\supp(\sigma)$-extension of $J$. By Definition \ref{def:supsem}, there is a $\sigma$-extension {$E^\prime$} of $\flat(J)$ such that {$E=E^\prime \cap Ar$}. $E$ is conflict-free because {$E=E^\prime \cap Ar$} and $E^\prime$ is conflict-free.
\end{proof}

\begin{example}
 By Lemmas \ref{lem:dungjsbafsem} and \ref{lem:jsbafcfreesem}, $\supp(\textit{complete})$, $\supp(\textit{stable})$, $\supp(\textit{grounded})$ and $\supp(\textit{preferred})$ are deductive, conflict-free JSBAF semantics.
\end{example}

\section{Deductive ASPIC$-$ and the Rationality Postulates}
\label{sec:postulates}

In this section we show how the deductive joint support relation can be applied to structured argumentation in order to satisfy the three rationality postulates (closure, direct and indirect consistency) in the context of unrestricted rebuttal. For this purpose, we define \emph{Deductive ASPIC$-$} (abbreviated as \emph{DA$-$}), show that it satisfies these three postulates and illustrate the result with an example.


In Deductive ASPIC$-$, arguments and attacks 
are defined in the same way as in ASPIC$-$, but we 
also
define a deductive joint support relation between arguments:



\begin{definition}
\label{def:JDA}
 Let $AS$ be an argumentation system. The \emph{JSBAF corresponding to $AS$ according to Deductive ASPIC$-$}, denoted by $\AFdam(AS)$, is the JSBAF ${(\Ar,\att',\Rightarrow')}$, where $\Ar$ is the set of arguments on the basis of $AS$, $A \att' B$ holds iff $A$ undercuts or unrestrictedly rebuts $B$, and $S \Rightarrow' A$ holds iff 
 $A$ is of the form ${S \mapsto \varphi}$ for some formula $\varphi$.
\end{definition}


\begin{definition}
\label{def:DAconc}
 Let ${AS = (R_s,R_d,n)}$ be an argumentation system and let $\sigma$ be a JSBAF semantics. For every $\sigma$-extension $E$ of $\AFdam(AS)$, the set $\{\Conc(A) \mid A \in E\}$ is called a \emph{set of DA$-$ conclusions of $AS$ under $\sigma$}.
\end{definition}

We now show that Deductive ASPIC$-$ satisfies closure, direct consistency and indirect consistency under any deductive, conflict-free JSBAF semantics.


\begin{lemma}
	Let $\sigma$ be a deductive, conflict-free JSBAF semantics. Then Deductive ASPIC$-$ satisfies closure under $\sigma$.
\label{lem:closure}
\end{lemma}

\begin{proof}
	Let $AS = (R_s,R_d,n)$ be an argumentation system and let $C$ be a set of DA$-$ conclusions of $AS$ under $\sigma$. Let ${S\mapsto\varphi}$ be a strict rule in $R_{s}$ with $S \subseteq C$. We need to show that $\varphi \in C$. 
	By Definition~\ref{def:DAconc}, there is a $\sigma$-extension $E$ of $\AFdam(AS) = {(Ar,\rightarrow',\Rightarrow')}$ such that $C = \{\Conc(A) \mid A \in E\}$. Since $S \subseteq C$, there is a set ${F\subseteq E}$ such that $S=\{\Conc(A) \mid A \in F\}$. Then $F \mapsto \varphi$ is an argument on the basis of $AS$. By Definition \ref{def:JDA}, $F \Rightarrow' {(F \mapsto \varphi)}$. Since $E$ is a $\sigma$-extension for a deductive JSBAF semantics $\sigma$, $F\subseteq E$ and $F\Rightarrow' {(F \mapsto \varphi)}$ imply that ${(F \mapsto \varphi) \in E}$. So $\varphi = {\Conc(F \mapsto \varphi)} \in {\{\Conc(A) \mid A \in E\} = C}$.
\end{proof}

\begin{lemma}
	Let $\sigma$ be a deductive, conflict-free JSBAF semantics. Then Deductive ASPIC$-$ satisfies direct consistency under $\sigma$.
\label{lem:dcon}
\end{lemma}

\begin{proof}
	Let $AS = (R_s,R_d,n)$ be a consistent argumentation system and let $C$ be a set of DA$-$ conclusions of $AS$ under $\sigma$. Suppose for a contradiction that ${\varphi, \neg \varphi \in C}$.
	By Definition~\ref{def:DAconc}, there is a $\sigma$-extension $E$ of $\AFdam(AS) = (Ar,\rightarrow',\Rightarrow')$ such that $C = \{\Conc(A) \mid A \in E\}$. Then there are ${F,F^\prime\in E}$ such that ${\Conc(F)=\phi}$ and ${\Conc(F^\prime)=\neg\phi}$. 
	By the consistency of $AS$, $F$ and $F^\prime$ cannot both be strict. Without loss of generality, assume $F^\prime$ is not strict. Then $F$ unrestrictedly rebuts $F^\prime$, which contradicts the conflict-freeness of $E$.
\end{proof}


Lemmas~\ref{lem:closure} and \ref{lem:dcon} together imply the following lemma about indirect consistency:

\begin{lemma}
	Let $\sigma$ be a deductive, conflict-free JSBAF semantics. Then Deductive ASPIC$-$ satisfies indirect consistency under $\sigma$.
\label{lem:indcon}
\end{lemma}


These three lemmas together with Lemmas \ref{lem:dungjsbafsem} and \ref{lem:jsbafcfreesem} imply the following theorem:

\begin{theorem}
	Let {$\sigma$} be an admissibility-based argumentation semantics. Then Deductive ASPIC$-$ satisfies closure, direct consistency and indirect consistency under $\supp(\sigma)$.
\end{theorem}


We now illustrate the functioning of Deductive ASPIC$-$ on an example, which is based on the example that Caminada~\cite{caminada2017rationality} used to show that closure is not satisfied under preferred semantics in ASPIC$-$. We show how the same example interpreted in Deductive ASPIC$-$ does satisfy closure.

\begin{example} \label{exp:tandem}
	Suppose Harry, Sally and Tom want to go on a bicycle ride with a tandem. Since the tandem only has two seats, only two of the three can ride it at a time. To formalize this scenario, we use the language $L = \{hw,sw,tw,ht,st,tt\}$, where $hw$ means ``Harry wants to ride the tandem'', $ht$ means ``Harry will ride the tandem'', and analogously for Sally ($sw$, $st$) and Tom ($tw$, $tt$).
	The scenario can be represented by an argumentation system ${AS= \{R_s,R_d,n\}}$, where
	$R_s=\{\mapsto hw; \mapsto sw; \mapsto tw; ht,st\mapsto \neg{tt};$ $st,tt\mapsto \neg{ht}; tt,ht\mapsto \neg{st}\}$,
	${R_d=\{hw\Mapsto ht; sw\Mapsto st; tw\Mapsto tt\}}$ and $n$ is empty. Intuitively, the strict rules $\mapsto hw$, $\mapsto sw$, and $\mapsto tw$ represent that all three of them want to ride on the tandem, the strict rules $ht,st\mapsto \neg{tt}$, $st,tt\mapsto \neg{ht}$ and $tt,ht\mapsto \neg{st}$ represent that the tandem can only seat two people, and the defeasible rules $hw\Mapsto ht$, $sw\Mapsto st$ and $tw\Mapsto tt$ represent that as far as possible each person gets to do what they want to do.
	
	The following arguments on the basis of $AS$ can be constructed:
	\onlypaper{\vspace{-3mm}}
	\begin{multicols}{3}
		\begin{itemize}
			\item ${A_1: \mapsto hw}$
			\item ${A_2: \mapsto sw}$
			\item ${A_3: \mapsto tw}$
			\item ${A_4: A_1\Mapsto ht}$
			\item ${A_5: A_2\Mapsto st}$
			\item ${A_6: A_3\Mapsto tt}$	
			\item ${A_7: A_5,A_6\mapsto \neg{ht}}$
			\item ${A_8: A_6,A_4\mapsto \neg{st}}$
			\item ${A_9: A_4,A_5\mapsto \neg{tt}}$
		\end{itemize}
		\end{multicols}
	\onlypaper{\vspace{-2mm}}
		
	The JSBAF $\AFdam(AS)$ is depicted in Figure~\ref{fig:tandemfw}. Its flattening $\flat(\AFdam(AS))$ is depicted in Figure~\ref{fig:tandemflatfw}. 
The preferred extensions of $\flat(\AFdam(AS))$ are as follows:	
\onlypaper{\vspace{-2mm}}
\begin{itemize}
 \item $E1^\prime = \{A_1,A_2,A_3,A_9,\bar{A_6},A_4,A_5,e(A_5,A_7),e(A_4,A_8)\}$
 \item $E2^\prime = \{A_1,A_2,A_3,A_8,\bar{A_5},A_4,A_6,e(A_6,A_7),e(A_4,A_9)\}$
 \item $E3^\prime = \{A_1,A_2,A_3,A_7,\bar{A_4},A_6,A_5,e(A_6,A_8),e(A_5,A_9)\}$
\end{itemize}
\onlypaper{\vspace{-1mm}}
\onlypaper{Given the complexity of $\flat(\AFdam(AS))$, we have provided a proof that these three extensions are the only preferred extensions of $\flat(\AFdam(AS))$ in a technical report~\cite{report}.}
\onlyreport{Given the complexity of $\flat(\AFdam(AS))$, we provide a proof below that these three extensions are the only preferred extensions of $\flat(\AFdam(AS))$ (Proposition~\ref{prop:E123}).}

From this it follows that the $\supp(\textnormal{preferred})$-extensions of $\AFdam(AS)$ are $\{A_1,A_2,$ $A_3,A_9,A_4,A_5\}$, $\{A_1,A_2,A_3,A_8,A_4,A_6\}$ and $\{A_1,A_2,A_3,A_7,A_6,A_5\}$. Therefore the sets of DA$-$ conclusions under the $\supp(\textnormal{preferred})$-semantics are $\{hw,sw,tw,\neg tt,ht,st\}$, $\{hw,sw,tw,\neg st,ht,tt\}$ and $\{hw,sw,tw,\neg ht,tt,st\}$.
The reader can easily verify that each set of DA$-$ conclusions under preferred semantics 
is closed under the set of strict rules $R_s$, in line with the closure postulate.
\end{example}
\onlypaper{\vspace{-4mm}}
%
%
%
%
	\begin{figure}[H]
		\centerfloat
		\scalebox{0.70}[0.70]{
			\begin{tikzpicture}[->,>=stealth,shorten >=1pt,auto,node distance=1.5cm,thick,main node/.style={circle,draw,font=\small}]
				\node[main node] (1) at (-5.2,0){$A_1$};
				\node[main node] (2) at (5.2,0) {$A_2$};
				\node[main node] (3) at (0,-4.7) {$A_3$};
				\node[main node] (4) at (0,0.8) {$A_4$};
				\node[main node] (5) at (1.96,-2.68) {$A_5$};
				\node[main node] (6) at (-1.96,-2.68) {$A_6$};	
				\node[main node] (7) at (0,3.6) {$A_7$};
				\node[main node] (8) at (4.4,-4) {$A_8$};
				\node[main node] (9) at (-4.4,-4) {$A_9$};
				\coordinate (10) at (-0.8,0.8) {};
				\coordinate (11) at (-2,-3.44) {};
				\coordinate (12) at (2.72,-2.4) {};
				\coordinate (13) at (-6.8,0) {};
				\coordinate (14) at (6.8,0) {};
				\coordinate (15) at (1.6,-4.7) {};
				\draw[<->,>=latex] (7) to (9);
				\draw[<->,>=latex] (8) to (9);
				\draw[<->,>=latex] (8) to (7);
				\draw[<->,>=latex] (6) to (9);
				\draw[<->,>=latex] (8) to (5);
				\draw[<->,>=latex] (7) to (4);
				\draw[double distance=1pt,-,>=latex] (6) to (10);
				\draw[double distance=1pt,-,>=latex] (5) to[out=180,in=-90] (10);
				\draw[double distance=1pt,->,>=latex] (10) to (7);
				\draw[double distance=1pt,-,>=latex] (5) to (11);
				\draw[double distance=1pt,-,>=latex] (4) to[out=-60,in=30] (11);
				\draw[double distance=1pt,->,>=latex] (11) to (9);
				\draw[double distance=1pt,-,>=latex] (6) to[out=60,in=150] (12);
				\draw[double distance=1pt,-,>=latex] (4) to (12);
				\draw[double distance=1pt,->,>=latex] (12) to (8);
				\draw[double distance=1pt,->,>=latex] (13) to (1);
				\draw[double distance=1pt,->,>=latex] (14) to (2);
				\draw[double distance=1pt,->,>=latex] (15) to (3);	
		\end{tikzpicture} 
		}
	\onlypaper{\vspace{-1mm}}
	\caption{The JSBAF $\AFdam(AS)$}
	\label{fig:tandemfw}
\end{figure}
\onlypaper{\vspace{-4mm}}


\begin{figure} [H]
\onlypaper{\vspace{-7.9mm}}
	\centerfloat
	\scalebox{0.69}[0.69]{
		\begin{tikzpicture}[->,>=stealth,shorten >=1pt,auto,node distance=1.5cm,thick,main node/.style={circle,draw,font=\small}]
				\node[main node] (e59) at (-1.55,2.64){$e(A_5,A_9)$};
				\node[main node] (e68) at (1.55,2.64){$e(A_6,A_8)$};
				\node[main node] (e57) at (-3.1,0){$e(A_5,A_7)$};
				\node[main node] (e67) at (3.1,0){$e(A_6,A_7)$};
				\node[main node] (e48) at (-1.55,-2.64){$e(A_4,A_8)$};
				\node[main node] (e49) at (1.55,-2.64){$e(A_4,A_9)$};
				\node[main node] (a6) at (-4.7,-2.64){$A_6$};
				\node[main node] (a5) at (4.7,-2.64) {$A_5$};
				\node[main node] (a4) at (0,5.4) {$A_4$};
				\node[main node] (a6p) at (-2.7,-1.5) {$\bar{A_6}$};
				\node[main node] (a5p) at (2.7,-1.5) {$\bar{A_5}$};
				\node[main node] (a4p) at (0,3) {$\bar{A_4}$};	
				\node[main node] (a7) at (0,8.3){$A_7$};
				\node[main node] (a9) at (-7.3,-4.2) {$A_9$};
				\node[main node] (a8) at (7.3,-4.2) {$A_8$};
				\node[main node] (a1) at (-5,2.64) {$A_1$};
				\node[main node] (a2) at (5,2.64) {$A_2$};
				\node[main node] (a3) at (0,-4.9) {$A_3$};		
				\draw[->,>=latex] (e48) to (a6);
				\draw[->,>=latex] (e49) to (a5);
				\draw[->,>=latex] (e57) to (a6);
				\draw[->,>=latex] (e67) to (a5);
				\draw[->,>=latex] (a4) to (a4p);
				\draw[->,>=latex] (a5) to (a5p);
				\draw[->,>=latex] (a6) to (a6p);
				\draw[->,>=latex] (e59) to (a4);
				\draw[->,>=latex] (e68) to (a4);
				\draw[<->,>=latex] (a8) to (a7); 
				\draw[<->,>=latex] (a8) to (a9); 
				\draw[<->,>=latex] (a9) to (a7);
				\draw[<->,>=latex] (a7) to (a4);
				\draw[<->,>=latex] (a9) to (a6);
				\draw[<->,>=latex] (a8) to (a5);	
				\draw[->,>=latex] (a4p) to[out=240,in=90] (e48);
				\draw[->,>=latex] (a4p) to[out=-60,in=90] (e49);
				\draw[->,>=latex] (a6p) to[out=0,in=210] (e67);
				\draw[->,>=latex] (a6p) to[out=60,in=210] (e68);
				\draw[->,>=latex] (a5p) to[out=120,in=-30] (e59);
				\draw[->,>=latex] (a5p) to[out=180,in=-30] (e57);
				\draw[->,>=latex] (a9) to[out=45,in=210] (e59);
				\draw[->,>=latex] (a9) to[out=15,in=210] (e49);
				\draw[->,>=latex] (a8) to[out=135,in=-30] (e68);
				\draw[->,>=latex] (a7) to[out=-105,in=90] (e57);
				\draw[->,>=latex] (a7) to[out=-75,in=90] (e67);
				\draw[->,>=latex] (a8) to[out=165,in=-30] (e48);	
		\end{tikzpicture}
		}
	\onlypaper{\vspace{-2mm}}
	\caption{Flattening $\flat(\AFdam(AS))$ of the JSBAF $\AFdam(AS)$}
	\label{fig:tandemflatfw}
	\onlypaper{\vspace{-2mm}}
\end{figure}


\onlyreport{
\begin{proposition}
\label{prop:E123}
	\textbf{The preferred extensions of $\flat(\AFdam(AS))$ are precisely $E1, E2$ and $E3$.}
\end{proposition}	
			\begin{proof}
			It is easy to see that each of $E_1$, $E_2$ and $E_3$ attacks every argument outside the extension, so each of them is a stable extension and thus also a preferred extension.
			
		 		Let us suppose there exists an extension $E$ of ${Flat(AF_{tandem})}$ such that $E \neq E1$, $E \neq E2$ and $E \neq E3$. Straightaway we see that ${E \neq \emptyset}$ because a preferred extension is a subset-maximal set. Therefore, $E$ must contain atleast an element $a$.
		 		
				We have structured the rest of the proof with case-wise distinctions for the different possible values of $a$. Since every preferred extension is an admissible set, we use the notion of admissibility henceforth.
				
				\noindent \underline{Case 1}: ${a = A_9}$ or ${a = A_8}$ or ${a = A_7}$. Without loss of generality we assume that  ${a = A_9}$. So, ${A_9 \in E}$. Then ${A_7 \notin E}$ and ${A_8 \notin E}$ since $E$ is conflict-free. Following the same line of reasoning, ${A_6 \notin E}$, ${e(A_5,A_9) \notin E}$ and ${e(A_4,A_9) \notin E}$. ${\bar{A_6} \in E}$ since it is defended by $A_9$. Thus, ${e(A_6,A_8) \notin E}$ and ${e(A_6,A_7) \notin E}$. ${A_4 \in E}$ and ${A_5 \in E}$ since they are defended by $A_9$. This implies, ${\bar{A_4} \notin E}$ and ${\bar{A_5} \notin E}$ because $E$ is conflict-free. Lastly, ${e(A_5,A_7) \in E}$ and ${e(A_4,A_8) \in E}$ as they are both defended by $E$. We arrive at $E = E1$.\\ Similarly, assuming ${a = A_8}$ leads us to $E = E_2$ and ${a = A_7}$ to $E = E_3$. \textit{Contradiction}.
				
				\noindent \underline{Case 2}: Case 1 does not hold and either ${a = A_6}$ or ${a = A_5}$ or ${a = A_4}$. Without loss of generality, let us assume ${a = A_6}$ and so ${A_6 \in E}$. Then ${\bar{A_6} \notin E}$, ${e(A_5,A_7) \notin E}$, ${e(A_4,A_8) \notin E}$ since $E$ is conflict-free. This also implies ${\bar{A_5} \in E}$ as it is needed to defend $A_6$ from the attack of ${e(A_5,A_7)}$ and ${\bar{A_4} \in E}$ as it is needed to defend $A_6$ from ${e(A_4,A_8)}$. Now since ${\bar{A_5} \in E}$, so, ${e(A_5,A_9) \notin E}$. Then, ${e(A_6,A_8) \in E}$ since it is needed to defend $\bar{A_4}$ from the attack of $A_4$. This implies ${A_5 \in E}$ as it defends $e(A_6,A_8)$ from $A_8$. Since $E$ is conflict-free so ${\bar{A_5} \notin E}$. \textit{Contradiction}. A similar kind of reasoning can be applied when ${a = A_5}$ or ${a = A_4}$.
				
				\noindent \underline{Case 3}: Case 1 and 2 do not hold and either ${a = \bar{A_6}}$ or ${a = \bar{A_4}}$ or ${a = \bar{A_5}}$. Without loss of generality, let us assume ${a = \bar{A_6}}$ and so ${\bar{A_6} \in E}$. Then, ${e(A_5,A_7) \in E}$ or ${e(A_4,A_8) \in E}$ in order to defend $\bar{A_6}$ from $A_6$. Without loss of generality let us assume ${e(A_5,A_7) \in E}$. So $A_5 \in E$ since it is needed to defend ${e(A_5,A_7)}$ from ${\bar{A_5}}$. However, in this case we have assumed that Case 2 does not hold. \textit{Contradiction}. A similar kind of reasoning can be applied when ${a = \bar{A_4}}$ or ${a = \bar{A_5}}$.

				\noindent \underline{Case 4}: Cases 1 and 2 and 3 do not hold. Then $a$ must be an argument of the form $e(b,c)$. Without loss of generality, let us assume ${a = e(A_5,A_7)}$ and so, ${e(A_5,A_7) \in E}$. This means that ${A_5 \in E}$ since it is needed to defend ${e(A_5,A_7)}$ from the attack of ${\bar{A_5}}$. \textit{Contradiction}.
				
  				In each of the four cases we have derived a contradiction. So our original assumption must be false. In other words, $E$ must be equal to either $E1$ or $E2$ or $E3$.		
			\end{proof}
}

\section{Conclusion and Future Work}
\label{sec:conclusion}

Caminada~\cite{caminada2017rationality} has established that ASPIC+, which uses restricted rebuttal, satisfies the three rationality postulates defined in Section~2 under any of the standard admissibility- based semantics, whereas ASPIC$-$, which uses unrestricted rebuttal, satisfied closure and indirect consistency only under the grounded semantics. In this paper we defined a modification of ASPIC$-$ called Deductive ASPIC$-$, which also uses unrestricted rebuttal, but which satisfies all three rationality postulates under any admissibility-based semantics. This is attained by keeping track not only of the attack relation between arguments, but also of the deductive joint support relation between arguments linked by an application of a strict rule.


The methodology introduced in this paper opens up multiple avenues for future research. First, the results presented in this paper have been limited to structured argumentation without preferences, so future work should study how these results could be generalized to a variant of Deductive ASPIC$-$ with preferences. 

Furthermore, while the results presented in this paper are limited to admissibility- based semantics, the general methodology is also applicable to naive-based semantics like CF2, SCF2, stage and stage2. So far, the application of these semantics to structured argumentation was limited by the fact that the closure postulate is violated under these semantics, even when restricted rebuttal is used. Given that empirical cognitive studies have found CF2 and SCF2 to be good models of human argument evaluation (see \cite{cramer2018empirical,cramer2019empirical,cramer2019scf2}), it seems to us to be a very worthwhile endeavor to attempt to remedy this situation. However, the approach of using the $sup$ operator, i.e.\ to use JSBAF semantics like $sup(\textnormal{CF2})$ or $sup(\textnormal{SCF2})$, will not yield to satisfaction of the closure postulate. Instead, one can adapt these naive-based argumentation semantics to JSBAF semantics in a different way. For example, a JSBAF variant of CF2 could be defined by ensuring the deductiveness property on the level of each SCC. 
Additionally, the definition of SCC would have to be adapted to account for the effect of deductive joint support (the paths required in the definition of SCCs should be able to pass thorough the support relation as well, however in the backward direction). This way the closure postulate can be made to be satisfied in combination with these naive-based semantics.

Another avenue for future research is to apply the methodology introduced in this paper to tackle the rationality postulates of non-interference and crash resistance (see \cite{caminada2017rationality}). Wu and Podlaszewski~\cite{wu2015implementing} have introduced an approach to satisfying these postulates by deleting inconsistent arguments, but when preferences are taken into account, this approach fails to satisfy closure. Combining their approach with ours yields a framework in which closure as well as non-interference and crash resistance can be satisfied in the presence of preferences.
\onlypaper{\vspace{-3.48mm}}

\bibliography{bibliography}
\bibliographystyle{plain}

\end{document}